\documentclass{article}

\usepackage[utf8]{inputenc} 
\usepackage[T1]{fontenc}    
\usepackage{hyperref}       
\usepackage[a4paper,top=3cm,bottom=2cm,left=3cm,right=3cm,marginparwidth=1.75cm]{geometry}
\usepackage{url}            
\usepackage{booktabs}       
\usepackage{amsfonts}       
\usepackage{nicefrac}       
\usepackage{microtype}      
\usepackage{amsmath}
\usepackage{graphicx}
\usepackage{algorithm,algorithmic}
\usepackage[colorinlistoftodos]{todonotes}
\usepackage{amssymb,amsthm}
\usepackage{amsfonts}
\usepackage{float}
\usepackage{tabularx}
\usepackage{mathtools}
\usepackage{authblk}

\newcommand{\N}{\mathbb{N}}
\newcommand{\Q}{\tilde{Q}}
\newcommand{\s}{\tilde{s}}
\newcommand{\tk}{\tilde{k}}

\newcommand{\sgn}{\mathrm{sign}}

\DeclarePairedDelimiter\floor{\lfloor}{\rfloor}

\newtheorem{thm}{Theorem}

\graphicspath{{images/}}

\begin{document}

\title{Quantization and Training of Low Bit-Width Convolutional Neural Networks for Object Detection}

%

\author[1]{Penghang~Yin}
\author[2]{Shuai~Zhang\thanks{Equal contribution.}}
\author[2]{Yingyong~Qi}
\author[2]{Jack~Xin} 
\affil[1]{\footnotesize Department of Mathematics, University of California, Los Angeles. Email: yph@ucla.edu.}
\affil[2]{\footnotesize Department of Mathematics, University of California, Irvine. Email: (szhang3, jack.xin, yqi)@uci.edu. }
\date{}
\maketitle
\begin{abstract}
We present LBW-Net, an efficient optimization based method for quantization and training of the low bit-width convolutional neural networks (CNNs). Specifically, we quantize the weights to zero or powers of two by minimizing the Euclidean distance between full-precision weights and quantized weights during backpropagation. We characterize the combinatorial nature of the low bit-width quantization problem. For 2-bit (ternary) CNNs, the quantization of $N$ weights can be done by an exact formula in $O(N\log N)$ complexity. When the bit-width is three and above, we further propose a semi-analytical thresholding scheme with a single free parameter for quantization that is computationally inexpensive. The free parameter is further determined by network retraining and object detection tests. LBW-Net has several desirable advantages over full-precision CNNs, including considerable memory savings, energy efficiency, and faster deployment. Our experiments on PASCAL VOC dataset show that compared with its 32-bit floating-point counterpart, the performance of the 6-bit LBW-Net is nearly lossless in the object detection tasks, and can even do better in some real world visual scenes, while empirically enjoying more than 4$\times$ faster deployment.
\end{abstract}

\section{Introduction}
Deep convolutional neural networks (CNNs) have demonstrated superior performance in various
computer vision tasks \cite{rfcn_16, imagenet_12, dl_15,zipcode_89,mnist_98,ssd_15,faster_rcnn_15,vgg_14,deeper_15}. However deep CNNs typically have hundreds of
millions of trainable parameters which easily take up hundreds of megabytes of memory, and billions
of FLOPs for a single inference. This poses a significant challenge for the deployment of deep CNNs
on small devices with limited memory storage and computing power such as mobile phones. To address this issue, recent efforts have been made to compress the model size and train neural networks with heavily quantized weights, activations, and gradients \cite{bc_15, bnn_16,surgery_16,tq_16,qnn_16,twn_16,xnor_16,entropy_17,inq_17,dorefa_16,ttq_16}, which demand less storage and fewer FLOPs for deployment. These models include BinaryConnect\cite{bc_15}, BinaryNet\cite{bnn_16}, XNOR-Net \cite{xnor_16}, TWN \cite{twn_16}, TTQ \cite{ttq_16}, DoReFa-Net\cite{dorefa_16} and QNN \cite{qnn_16}, to name a few. In particular, binary (1-bit) and ternary (2-bit) weight models not only enable high model compression rate, but also eliminate the need of most floating-point multiplications during forward and backward propagations, which shows promise to resolve the problem. Compared with binary models, ternary weight networks such as TWN strike a better balance between compression rate and accuracy. It has been shown that ternary weight CNNs \cite{twn_16} can achieve nearly lossless accuracy on MNIST\cite{mnist_98} and CIFAR-10 \cite{cifar_09} benchmark datasets. Yet with fully ternarized weights, there is still noticeable drop in performance on larger datasets like ImageNet \cite{imagenet_09}, which suggests the necessity of relatively wider bit-width models with stronger model capacity for challenging tasks. 

An incremental network quantization strategy (INQ) is proposed in \cite{inq_17} for converting pre-trained full-precision CNNs into low bit-width versions whose weights are either zero or powers of two. A $b$ bit-width model can have $2^{b-1}+1$ distinct candidate values, in which 2 bits are used for representing the zero and the signs, while the remaining $b-2$ bits for the powers. More precisely, the parameters are constrained to $2^s\times\{0,\pm 2^{1-2^{b-2}}, \pm 2^{2-2^{b-2}}, \dots, \pm 1\}$ associated with a layerwise scaling factor $2^s$, $s$ an integer depending only on the weight maximum in the layer. At inference time, the original floating-point multiplication operations can be replaced by faster and cheaper binary bit shifting. The quantization scheme of \cite{inq_17} is however heuristic.

In this paper, we present the exact solution of the general $b$-bit approximation problem of a real weight vector $W^f$ in the least squares sense. If $b=2$ and the dimension of $W^f$ is $N$, the computational complexity of the 2 bit solution is $O(N\log N)$. At $b\geq 3$, the combinatorial nature of the solution renders direct computation too expensive for large scale tasks. We shall develop a semi-analytical quantization scheme involving a single adjustable parameter $\mu$ to set up the quantization levels. The exponent $s$ in the scaling factor can be calculated analytically from $\mu$ and the numbers of the downward sorted weight components between quantization levels. If the weight vector comes from a Gaussian ensemble, the parameter $\mu$ can be estimated analytically. However, we found that the weight vectors in CNNs (in particular ResNet) are strongly non-Gaussian. 
In this paper, $\mu$ is determined based on the object detection performance after retraining the network. This seems to be a natural choice in general as quantization is often part of a larger computer vision problem as is here. 
Therefore, the optimal parameter $\mu$ should not be decided by approximation (the least squares problem) errors alone. Indeed, we found that at $b\geq 4$, $\mu = \frac{3}{4} \|W^f\|_{\infty}$ gives the best detection performance, which suggests that a percentage of the large weights plays a key role in representing the image features and should be encoded during quantization. 

Network retraining is necessary after quantization as a way for the system to adjust and absorb the resulting errors. Besides warm start, INQ \cite{inq_17} requires a careful layerwise partitioning and grouping of the weights which are then quantized and re-trained incrementally group by group rather than having all weights updated at once. Due to both classification and detection networks involved in this work, we opted for a simpler retraining method, a variant of the projected stochastic gradient descent (SGD) method (see \cite{twn_16} and references therein). As a result, our LBW-Net can be trained either from scratch or a partial warm start. During each iteration, besides forward and backward propagations, only an additional low cost thresholding (projection) step is needed to quantize the full-precision parameters to zero or powers of two. We train LBW-Net with randomly initialized weights in the detection network (R-FCN in \cite{rfcn_16}), and pre-trained weights in ResNet \cite{resnet_15}. We conduct object detection experiments on PASCAL VOC data sets \cite{pascal_10} as in \cite{rfcn_16,faster_rcnn_15}. We found that at bit-width $b=6$, the accuracies of the quantized networks are well within $1\%$ of those of their 32-bit floating-point counterparts on both ResNet-50 and ResNet-101 backbone architectures. In some complex real world visual scenes, the 6-bit network even detects persons missed by the full-precision network. 

The rest of the paper is organized as follows. 
In section 2, we construct the exact solution of the general low bit-width approximation problem and present our semi-analytical quantization scheme with a single adjustable parameter $\mu$. We also outline the training algorithm and the choice of $\mu$.  
In section 3, we describe our experiments, the datasets, 
the object detection results, the non-Gaussian and sparsity properties of the floating weights in training. In section 4, we conclude with remarks on future work. 

\section{Training low bit-width convolutional neural networks}

\subsection{Weight quantization at low bit-width}

For general quantization problem, we seek to minimize the Euclidean distance between the given full-precision weight vector $W^f$ and quantized weight vector $W^q$, which is formulated as the following optimization problem:
$$
\min_{W^q} \|W^q - W^f\|^2 \quad \mbox{subject to} \quad  W^q\in\mathcal{Q},
$$
where $\mathcal{Q}$ is the set of quantized states. To quantize the full-precision weights into low-precision ones of $b$ bits ($b\geq 2$), we constrain the quantized weights to the value set of $2^s\times\{0,\pm 2^{1-n}, \pm 2^{2-n},\dots, \pm 1\}$ for some integer $s\in\mathbb{Z}$, where $n = 2^{b-2}$ and $2^s$ serves as the scaling factor. The minimal distance problem becomes:

\begin{equation}\label{model:prox}
(s^*, Q^*) = \arg\min_{s\in\mathbb{Z},Q} \|2^s Q - W^f\|^2 \quad \mbox{subject to} \quad  Q_i\in\{0, \pm 2^{1-n},\dots, \pm 1\}.
\end{equation}
Then the optimal quantized weight vector is given by $2^{s^*}Q^*$. A precise characterization of (\ref{model:prox}) is as follows.

\begin{thm}\label{thm:exact}
Let  $b\geq 2$, $n = 2^{b-2}$, and $k_0, \dots, k_{n-1}\in\N$. Suppose that $W^f_{[k_0]}$ keeps the $k_0$ largest components in magnitude of $W^f$ and zeros out the other components; $W^f_{[k_1]}$ extracts the next $k_1$ largest components and zeros out the other components, and so on. The solution $Q^*$ to (\ref{model:prox}) is: 
$$
Q^* = \sum_{t = 0}^{n-1}\sgn(W^f_{[k^*_t]})2^{-t},
$$
where 
\begin{equation}\label{model:optimal_k}
(k^*_0, \dots, k^*_{n-1}) = \arg\min_{k_0,\dots,k_{n-1}\in \N} \; g\left( \sum_{t=0}^{n-1} \|W_{[k_t]}^f\|_1 2^{-t}, \sum_{i=0}^{n-1} k_t 2^{-2t} \right) 
\end{equation}
with 

$$g(u,v):= v\left (2^{\floor{\log_2 \frac{4u}{3v}}}-\frac{u}{v}\right )^2 - \frac{u^2}{v}.$$ 
\medskip

The bracket in $g(u,v)$ is the floor operation. 
Moreover, the optimal power of scaling is: 

$$
s^*=\floor{\log_2 \frac{4 \sum_{t=0}^{n-1}2^{-t}\|W_{[k^*_t]}^f\|_1}{3\sum_{t=0}^{n-1} k^*_t 2^{-2t}}}.
$$
\end{thm}
In Theorem 1, we have assumed that the components of $W^f$ have no ties in 
magnitudes, as such situation occurs with zero probability for random floating vectors from continuous distributions. To solve the problem (\ref{model:prox}) by Theorem \ref{thm:exact}, we need to sort the elements of $W^f$ in magnitude, and find the optimal numbers of weights $k_0^*,\dots,k^*_{n-1}$ at $n$ quantization levels by solving (\ref{model:optimal_k}). We can then obtain the optimal scaling factor $2^{s^*}$. The largest $k_0^*$ weights (in magnitude) are quantized to $\pm2^{s^*}$, and the next largest $k_1^*$ weights to $\pm 2^{s^*-1}$, and so on. Finally, all the remaining small weights are pruned to 0.

The subproblem (\ref{model:optimal_k}) is intrinsically combinatorial however. In the simplest case $b=2$ with ternary weight networks, by Theorem \ref{thm:exact}, $k_0^* = \arg\min_{k_0\in\N} \; g(\|W^f_{[k_0]}\|_1,k_0)$, and the solution to (\ref{model:prox}) is given by $Q^* = \sgn(W^f_{[k_0^*]})$, $s^*=\floor*{\log_2 \frac{4\, \|W^f_{[k_0^*]}\|_1}{3\, k_0^*}}$.
Therefore, the weight ternarization mainly involves sorting magnitudes of the elements in $W^f$ and computing a cumulative sum of the sorted sequence, which requires a computational complexity of $O(N\log(N))$, where $N$ is number of entries in $W^f$. 
When $b>2$ and $n>1$, solving (\ref{model:optimal_k}) by direct enumeration  becomes computationally too expensive for large scale problems such as convolutional neural networks and thus impractical. Hereby we propose a low-cost approximation of $Q^*$, motivated by the  empirical quantization schemes in \cite{twn_16,inq_17}. To this end, by selecting a proper threshold value $\mu$, we set 
\medskip

\begin{equation}\label{approx}
\Q_i^* = 
\begin{cases}
0 & \mbox{if} \quad |W_i^f|< \frac{2^{2-n}}{3}\mu,\\
\sgn(W_i^f)2^{1-n} & \mbox{if} \quad \frac{2^{2-n}}{3}\mu \leq|W_i^f|< 2^{2-n}\mu \\
\sgn(W_i^f)2^{-t} & \mbox{if} \quad 2^{-t}\mu \leq|W_i^f|< 2^{-t+1}\mu, \; t = 1,\dots,n-2,\\
\sgn(W_i^f) & \mbox{if} \quad \mu \leq |W^f_i|.
\end{cases}
\end{equation}
\medskip

Note that the case $t = n-1$ in (\ref{approx}) needs special treatment because one of the neighboring quantized values is 0. The parameter $\mu$ is the only free parameter in (\ref{approx}).  

\begin{thm}\label{thm:approx}
The optimal power $\s^*$ of the scaling factor with respect to the approximate $\Q^*$ in (\ref{approx}) is

\begin{equation}\label{exactsf}
\s^* = \floor*{\log_2\frac{4\sum_{t=0}^{n-1} 2^{-t}\|W_{[\tk^*_t]}^f\|_1}{3\sum_{t=0}^{n-1} \tk^*_t 2^{-2t}}}.
\end{equation}
\medskip

Here $W_{[\tk^*_t]}$ is defined as in Theorem \ref{thm:exact}, and $\tk^*_t$ is the number of entries of $W^f$ in the $t$-th largest group according to the division of (\ref{approx}).
\end{thm}

We remark that the output of $\tilde{Q}^*$ consists of mostly the scaled signs, hence $\tilde{Q}^*$ resembles a ``phase factor''. On the other hand, the scaling factor $2^{\s^*}$ is the corresponding amplitude. Putting the two factors together, one can view the low bit-width weight approximation as an approximate polar decomposition of the real weight vector. The proof of Theorem 1 is in the appendix from which Theorem 2 follows.

\subsection{Training algorithm}
We used a projected SGD-like algorithm as in \cite{twn_16, xnor_16} for training LBW-Net. At each gradient-descent step, the minibatch gradient is evaluated at the quantized weights, and a scaled gradient is subtracted from the full-precision weights instead of the quantized weights per standard projected gradient method. The quantization is then done layer by layer by the formulas (\ref{approx}) and (\ref{exactsf}) with $\mu$ selected as $\frac{3}{4}\|W^f\|_\infty$ for each layer at bit-width 4 or above. To compute the optimal power $s^*$ in (\ref{exactsf}), we find it sufficient to use the partial sums $\sum_{t=0}^{3} 2^{-t}\|W_{[\tk^*_t]}^f\|_1$ and $\sum_{t=0}^{3} \tk^*_t 2^{-2t}$ instead, as the tail values are negligible. In addition, we adopted batch normalization \cite{bn_15}, adaptive learning rate, and Nesterov momentum \cite{nesterov_83} to promote training efficiency.

\section{Experiments} 
We implemented our LBW-Net with the R-FCN \cite{rfcn_16} structure on PASCAL VOC dataset which has $20$ object categories. Same as \cite{rfcn_16} , the training set is the union of VOC 2007 trainval and VOC 2012 trainval (“07+12”), and test results are evaluated on the VOC 2007 test set. So there are in total $16,551$ images with $40,058$ objects in the training set, and $4,952$ images in the test set. The performance of object detection is measured by mean Average Precision (mAP). \footnote{All mAP scores are computed with the Python version of the test codes provided by RCNN/Fast RCNN/Faster RCNN GitHub repositories.} Our experiments are carried out on Caffe \cite{caffe_14} with an Nvidia Titan X GPU under Linux system.

\subsection{R-FCN on PASCAL VOC}
We first employed ResNet-50 as the backbone network architecture for R-FCN. In the experiments, we tested $4, 5, 6$-bit LBW-Net and compared evaluation results with the corresponding 32-bit floating point models. For fair comparison, all these tests used the same initial weights, which are pre-trained convolutional feature maps from ResNet-50 while the weights in the other convolution layers are randomly initialized. A similar procedure is applied for experiments with ResNet-101. In \cite{entropy_17}, comparable results to ours were reported on ResNet-50 based detection. However, their method did not work on the deeper ResNet-101 based detection. Interestingly, although failed for ResNet-101 based detection, their approach succeeded in the classification task using Resnet-101, which suggests that quantization of detection networks is trickier in practice. 

In the R-FCN structure, there is no fully-connected layer. We quantized all convolutional layers with the same low bit-width quantization formula for each layer. 
\medskip

\begin{table}[ht]
\label{table: mAP}
\centering
\begin{tabularx}{.8\textwidth}{|c|X|c|X|}
  \hline			
  R-FCN, ResNet-50           & mAP & R-FCN, ResNet-101 & mAP \\ 
  \hline
  4-bit LBW  & $74.37\%$ 
& 4-bit LBW  & $76.79\%$ \\ 

  5-bit LBW  & $76.99\%$
& 5-bit LBW  & $77.83\%$   \\ 

  6-bit LBW  & $77.05\%$ 
& 6-bit LBW  & $78.24\%$   \\ 


  32-bit full-precision & $77.46\%$
& 32-bit full-precision & $78.94\%$  \\
  \hline  
\end{tabularx}
\bigskip
\caption{Object detection experiments on PASCAL VOC with R-FCN + ResNet-50/ResNet-101. Training set is VOC 07+12 trainval. The results are evaluated on VOC 07 test.}
\end{table}

\begin{figure}[H]
\begin{tabular}{cc}
\textbf{32-bit} & \textbf{6-bit} \\
\begin{minipage}[t]{0.45\linewidth}
\includegraphics[scale=0.2]{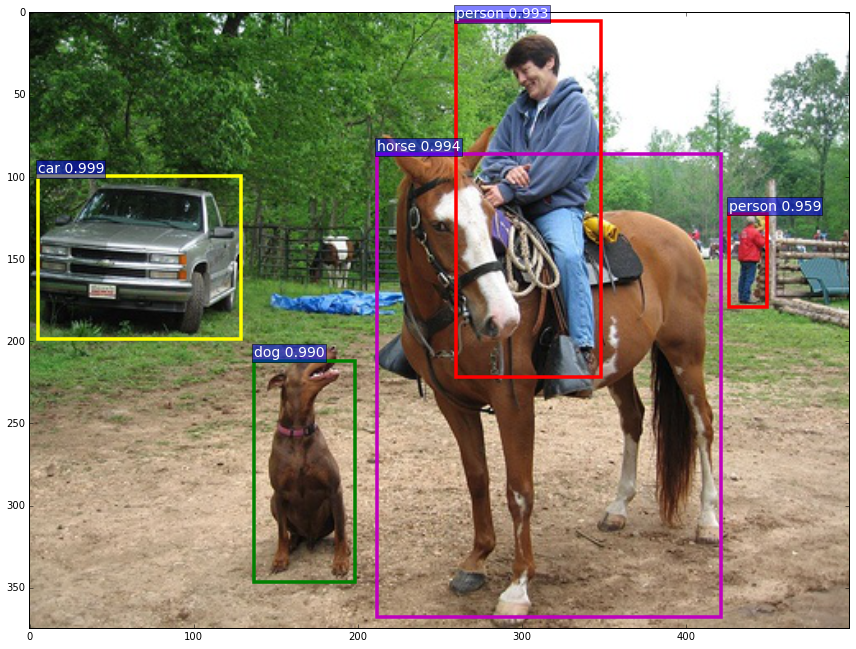}
\end{minipage}  &
\begin{minipage}[t]{0.45\linewidth}
\includegraphics[scale=0.2]{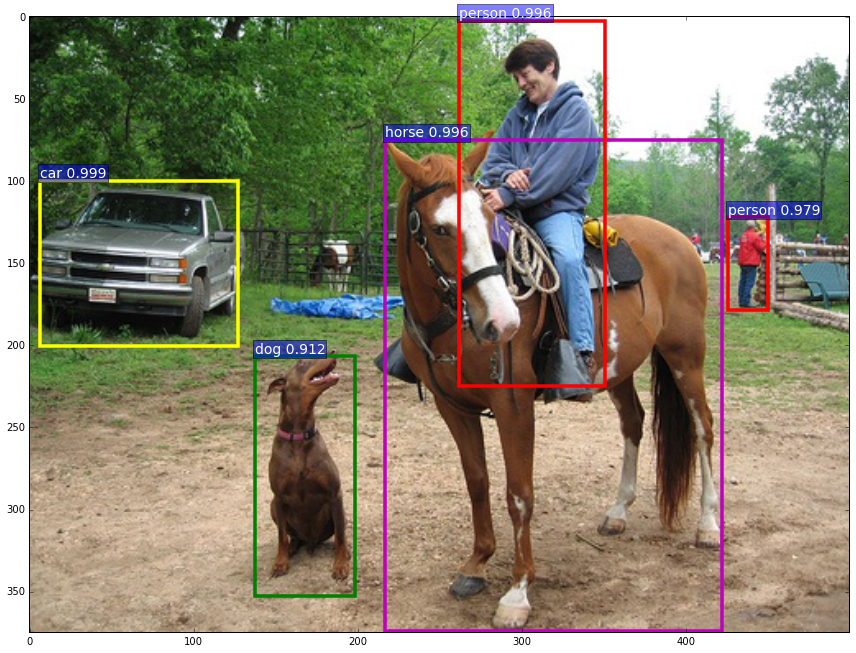}  
\end{minipage}  \\
\begin{minipage}[t]{0.45\linewidth}
\includegraphics[scale=0.2]{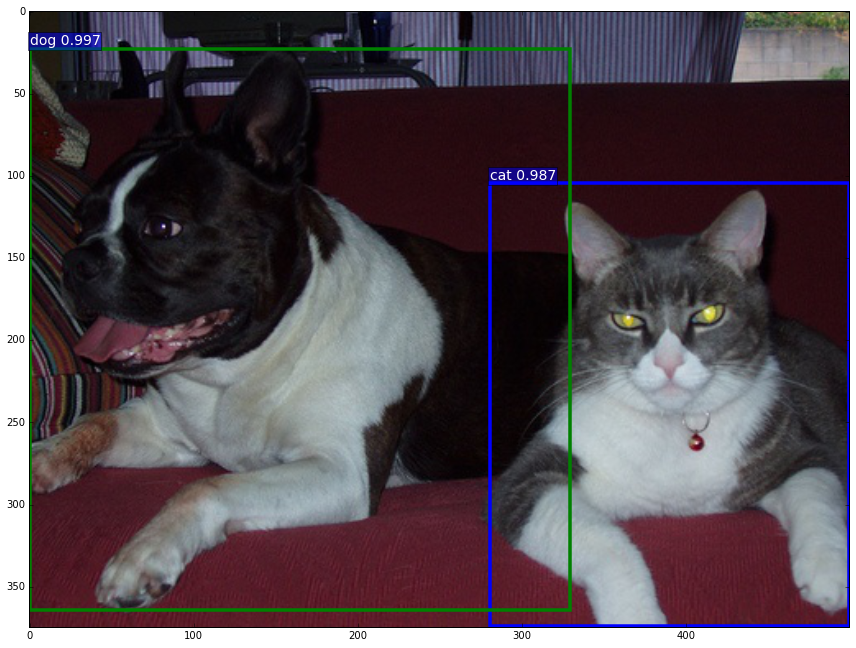}
\end{minipage}  &
\begin{minipage}[t]{0.45\linewidth}
\includegraphics[scale=0.2]{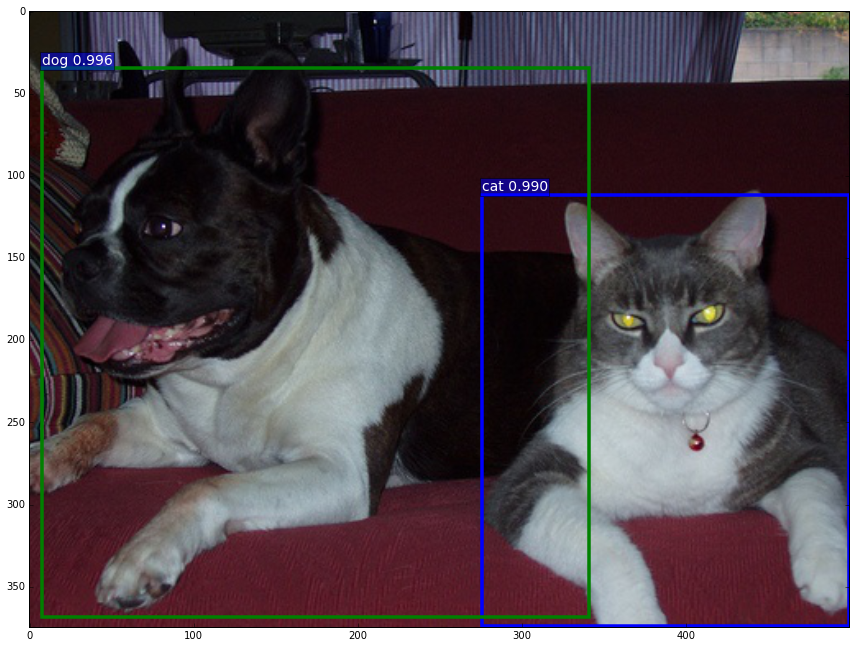}  
\end{minipage}  \\
\begin{minipage}[t]{0.45\linewidth}
\includegraphics[scale=0.2]{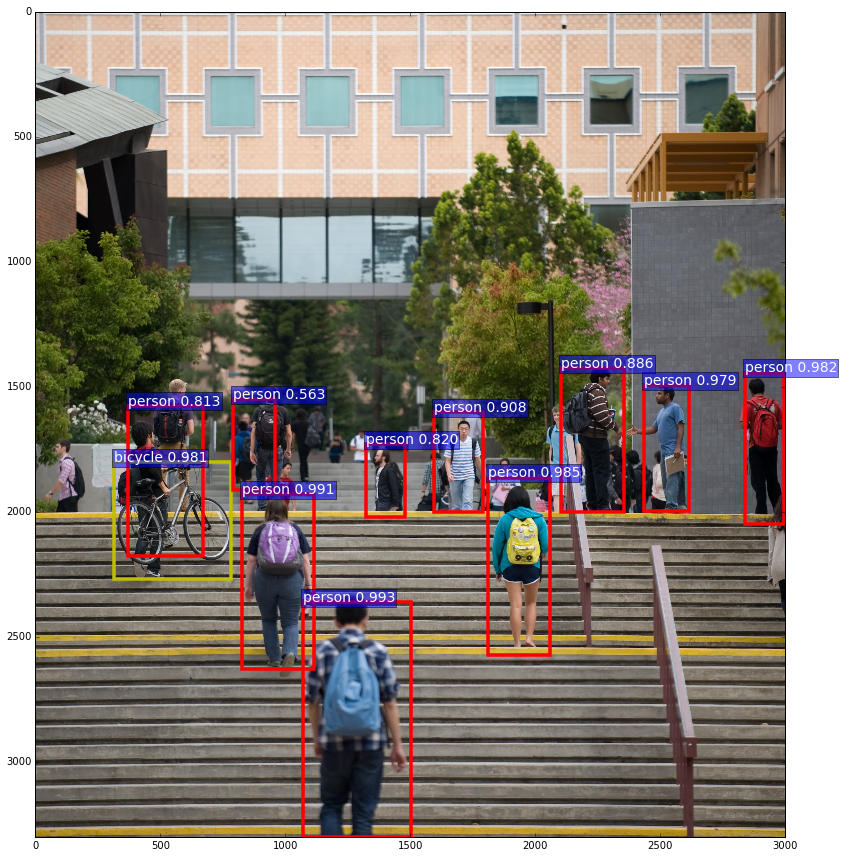}
\end{minipage}  &
\begin{minipage}[t]{0.45\linewidth}
\includegraphics[scale=0.2]{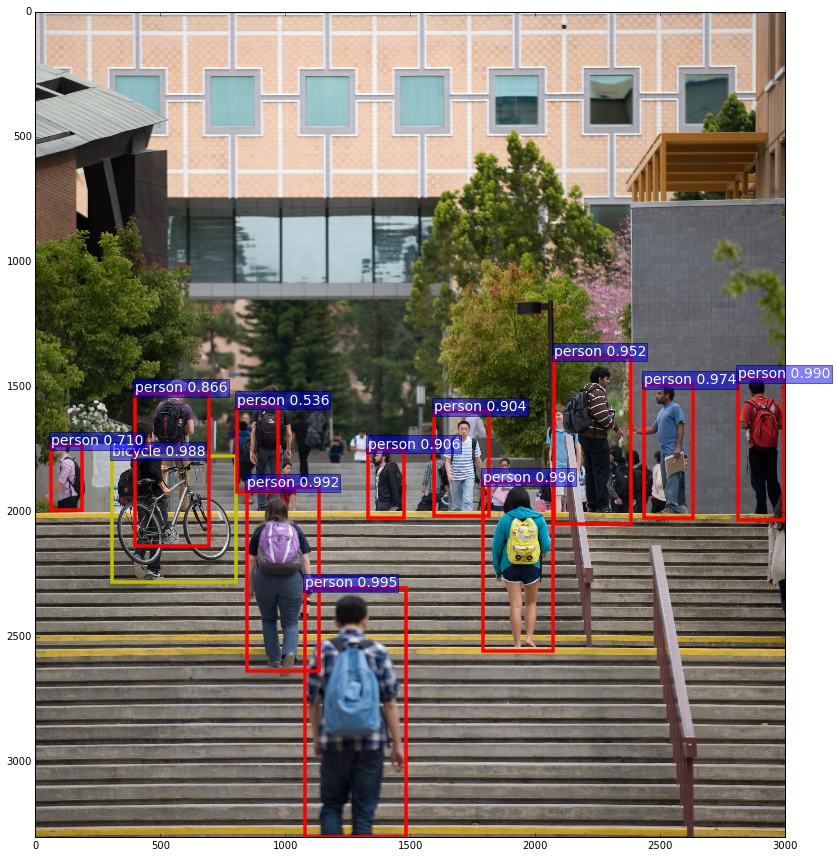}  
\end{minipage}  
\end{tabular}
\bigskip
\caption{Curated examples of 6-bit LBW detection results on 3 sample images, compared with those from the corresponding full precision model. The left columns are results of 32-bit full-precision model, while the right images come from 6-bit LBW model. The network is R-FCN + ResNet-50, and the training data is 2007+2012 trainval. The threshold value 0.5 is used for display.}
\label{figure: demo}
\end{figure}

Table 1 shows mAP results from our experiments. With larger bit-width, LBW models achieved higher mAP values, true for both R-FCN + ResNet-50 and R-FCN + ResNet-101. The models trained with the 6-bit LBW scheme almost approach the best mAP of 32-bit full precision models. Besides these quantitative measures, in Fig. \ref{figure: demo}, we illustrate detection accuracies using R-FCN + ResNet-50 via samples processed by 6-bit LBW in comparison with those by the `ground truth' full precision model. The first 2 photos are chosen from the 2007 Pascal VOC dataset and the third photo is taken at a university campus with a much more complicated visual scene. In the first 2 photos, both the 6-bit LBW and full precision models detected the major objects correctly, with nearly the same bounding box positions and high classification scores. In the third photo, the 6-bit LBW even surpassed the performance of the full precision model, by detecting a student at the very left side of the top staircase with a score of $0.710$. Also the 3rd student from the right (the student in the middle) on the top staircase is detected with a score of $0.952$ ($0.906$) by the 6 bit LBW vs. $0.886$ ($0.820$) by the full precision model. Interestingly, these three students are all side-viewed.    

At inference time, we have observed an immediate at least 4$\times$ speedup given by our 6-bit R-FCN model. For the three images shown in Fig. \ref{figure: demo}, the computing time are 0.507s, 0.441s, and 32.269s using a 32-bit R-FCN+ResNet-50 on GPU, while the costs are 0.098s, 0.106s and 6.113s respectively by our 6-bit counterpart. 

\subsection{Statistical Analysis of Weights}
In Fig. 2, we illustrate the weight distributions of two floating convolutional layers by histograms. The $p$-values of a standard hypothesis testing procedure in statistics on normality showed up very small (less than $10^{-5}$), indicating the strong non-Gaussian behavior of the floating weights in training. 
This phenomenon posed a challenge to the analytical effort of estimating the parameter $\mu$ in quantization using probability distribution functions as suggested for TWN \cite{resnet_15}. 

In Table 2 and Table 3
, 
we show the weight percentage distribution of two sample convolutional layers in R-FCN + ResNet50 between different magnitude levels of the quantization for low-bit width and full-precision models. The three low bit-width models involve truncation and encoding operations. The 6 bit-width columns appear to approach the 32-bit float columns on most rows. 
However, the percentages on the last three (two) rows under the  low-bit LBW models in Table 2 (3) are identical to each other and are much larger than the corresponding percentage in the full precision model. This shows that the trained low-bit LBW models captured rather well a small percentage of the large weights. In deep CNNs, the large magnitude weights occupy a small percentage yet have a significant impact on the model accuracy. That is why we chose the partition parameter $\mu$ to be near the maximum norm of the weights. 

It is worthwhile to note from the two tables that the 4-bit LBW can save lots of memory thanks to both low-bit weights and high sparsity. Over 82\% (58\%) of the weights are zeros in the convolutional residual block 
(RPN layer) of the R-FCN plus ResNet50 network. With the help of 'Mask' technology in circuit chip design, zero-valued weights will be skipped and the computational efficiency can be much improved. However, as shown in Table 1, the 4-bit LBW still suffers a few more percentages of accuracy loss than the 
5-bit and 6-bit models. 
The 6-bit LBW model approximates the feature representation capability of the full precision network the best with a sufficient number of smaller levels of quantized weights. For that reason, it almost recovers the performance of the full precision model on the test set. The 6-bit LBW model saves around $5.3\times$ weights memory with a small loss of accuracy. The memory savings and the near lossless accuracy of the 6-bit LBW 
may work well on a modern chip design where all multiplication operations in the convolutional layers can be replaced by bit-wise shift operations, thus highly improving the computing efficiency in applications.

\begin{figure}[ht]
\begin{tabular}{cc}
\textbf{Conv layer in residue block} & \textbf{RPN layer} [22] \\
\begin{minipage}[t]{0.5\linewidth}
\includegraphics[scale=0.5]{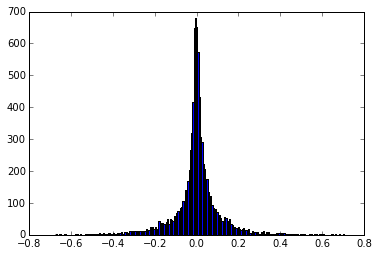}
\end{minipage}  &
\begin{minipage}[t]{0.5\linewidth}
\includegraphics[scale=0.5]{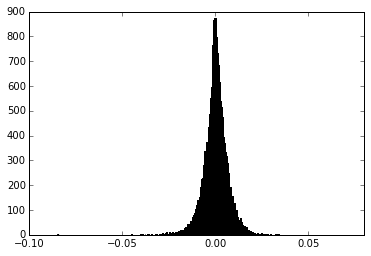}  
\end{minipage}  \\
$\text{Kurtosis} = 6.113$, \ $\text{Skewness} = -0.112$ 
& $\text{Kurtosis} = 9.398$, \ $\text{Skewness} = -0.481$ \\
\end{tabular}
\caption{Histograms of the float weights in 2 convolutional layers of 32-bit full-precision trained R-FCN + ResNet-50 model. For both of these 2 layers, the $p$-values of normal distribution hypothesis testing are extremely small, less than $10^{-5}$. Also the excess kurtosis measures are much larger than the value for normal distribution, which is $0$. Thus these weights are far from being normally distributed.}
\label{figure: hist}
\end{figure}

\begin{table}[ht]
\label{table: weights}
\centering
\begin{tabular}{|c|c|c|c|c|}
  \hline			
  R-FCN, ResNet-50 & 4-bit LBW & 5-bit LBW & 6-bit LBW & 32-bit full-precision \\
  \hline
  $|w| < 2^{-16}$              & $82.882\%$ & $10.072\%$ 
  & $0.030\%$ & $0$ \\ 
  \hline
  $2^{-16} \leq |w| < 2^{-15}$ & $0$        & $0$        
  & $0.060\%$ & $0.076\%$ \\
  \hline
  $2^{-15} \leq |w| < 2^{-14}$ & $0$        & $0$        
  & $0.141\%$ & $0.225\%$ \\
  \hline
  $2^{-14} \leq |w| < 2^{-13}$ & $0$        & $0$        
  & $0.233\%$ & $0.271\%$ \\
  \hline
  $2^{-13} \leq |w| < 2^{-12}$ & $0$        & $0$        
  & $0.486\%$ & $0.613\%$ \\
  \hline
  $2^{-12} \leq |w| < 2^{-11}$ & $0$        & $0$        
  & $0.922\%$ & $1.283\%$ \\
  \hline
  $2^{-11} \leq |w| < 2^{-10}$ & $0$        & $0$ 
  & $1.964\%$ & $2.610\%$ \\
  \hline
  $2^{-10} \leq |w| < 2^{-9}$  & $0$        & $0$ 
  & $3.776\%$ & $4.945\%$ \\
  \hline
  $2^{-9} \leq |w| < 2^{-8}$   & $0$        & $0$ 
  & $7.343\%$ & $9.524\%$ \\
  \hline
  $2^{-8} \leq |w| < 2^{-7}$   & $0$        & $18.392\%$ 
  & $13.509\%$ & $16.713\%$ \\
  \hline
  $2^{-7} \leq |w| < 2^{-6}$   & $0$        & $21.221\%$ 
  & $21.221\%$ & $23.581\%$ \\
  \hline
  $2^{-6} \leq |w| < 2^{-5}$   & $0$ 	    & $24.270\%$ 
  & $24.270\%$ & $22.993\%$ \\
  \hline
  $2^{-5} \leq |w| < 2^{-4}$   & $0$        & $17.706\%$ 
  & $17.706\%$ & $12.627\%$ \\
  \hline
  $2^{-4} \leq |w| < 2^{-3}$   & $15.479\%$ & $6.700\%$ 
  & $6.700\%$ & $3.784\%$ \\
  \hline
  $2^{-3} \leq |w| < 2^{-2}$   & $1.408\%$  & $1.408\%$ 
  & $1.408\%$ & $0.608\%$ \\
  \hline
  $2^{-2} \leq |w| < 2^{-1}$   & $0.228\%$  & $0.228\%$ 
  & $0.228\%$ & $0.098\%$ \\
  \hline
  $2^{-1} \leq |w|$            & $0.003\%$  & $0.003\%$ 
  & $0.003\%$ & $0$ \\
  \hline  
\end{tabular}
\medskip

\caption{Statistics of low-bit and full precision weights ($w$) of one convolutional residual block layer in R-FCN + ResNet-50 at different bit-widths. 
For 4, 5, 6-bit LBW models, the weights in the first row of partition are exactly equal to 0, and come from rounding down small floating weights during training.}
\end{table}

\begin{table}[ht]
\label{table: weights detection}
\centering
\begin{tabular}{|c|c|c|c|c|}
  \hline			
  R-FCN, ResNet-50 & 4-bit LBW & 5-bit LBW & 6-bit LBW & 32-bit full-precision \\
  \hline
  $|w| < 2^{-19}$              & $58.188\%$ & $4.000\%$ 
  & $0.016\%$ & $0.019\%$ \\ 
  \hline
  $2^{-19} \leq |w| < 2^{-18}$ & $0$        & $0$        
  & $0.031\%$ & $0.022\%$ \\
  \hline
  $2^{-18} \leq |w| < 2^{-17}$ & $0$        & $0$        
  & $0.047\%$ & $0.045\%$ \\
  \hline
  $2^{-17} \leq |w| < 2^{-16}$ & $0$        & $0$        
  & $0.095\%$ & $0.089\%$ \\
  \hline
  $2^{-16} \leq |w| < 2^{-15}$ & $0$        & $0$        
  & $0.185\%$ & $0.177\%$ \\
  \hline
  $2^{-15} \leq |w| < 2^{-14}$ & $0$        & $0$        
  & $0.370\%$ & $0.355\%$ \\
  \hline
  $2^{-14} \leq |w| < 2^{-13}$ & $0$        & $0$ 
  & $0.751\%$ & $0.714\%$ \\
  \hline
  $2^{-13} \leq |w| < 2^{-12}$  & $0$        & $0$ 
  & $1.501\%$ & $1.413\%$ \\
  \hline
  $2^{-12} \leq |w| < 2^{-11}$   & $0$        & $0$ 
  & $2.993\%$ & $2.836\%$ \\
  \hline
  $2^{-11} \leq |w| < 2^{-10}$   & $0$        & $7.949\%$ 
  & $5.952\%$ & $5.616\%$ \\
  \hline
  $2^{-10} \leq |w| < 2^{-9}$   & $0$        & $11.676\%$ 
  & $11.685\%$ & $11.061\%$ \\
  \hline
  $2^{-9} \leq |w| < 2^{-8}$   & $0$ 	    & $21.571\%$ 
  & $21.588\%$ & $20.625\%$ \\
  \hline
  $2^{-8} \leq |w| < 2^{-7}$   & $0$        & $31.553\%$ 
  & $31.539\%$ & $31.370\%$ \\
  \hline
  $2^{-7} \leq |w| < 2^{-6}$   & $39.837\%$ & $21.137\%$ 
  & $21.134\%$ & $23.257\%$ \\
  \hline
  $2^{-6} \leq |w| < 2^{-5}$   & $1.953\%$  & $2.093\%$ 
  & $2.091\%$ & $2.397\%$ \\
  \hline
  $2^{-5} \leq |w| < 2^{-4}$   & $0.022\%$  & $0.021\%$ 
  & $0.022\%$ & $0.004\%$ \\
  \hline
  $2^{-4} \leq |w|$            & $0.0001\%$  & $0.0001\%$ 
  & $0.0001\%$ & $0$ \\
  \hline  
\end{tabular}
\medskip

\caption{Statistics of low-bit and full precision weights ($w$) of one RPN layer in R-FCN + ResNet-50 at different bit-widths. For 4, 5, 6-bit LBW models, the weights in the first row of partition are exactly equal to 0, and come from rounding down small floating weights during training.}
\end{table}

\section{Concluding Remarks}
We discovered the exact solution of the general low-bit approximation problem of a real weight vector in the least squares sense, and proposed a low cost semi-analytical quantization scheme with a single adjustable parameter. This parameter is selected and optimized through training and testing on object detection data sets to approach the performance of the corresponding full precision model. The accuracy of our 6-bit width model is well-within 1\%  of the full precision model on PASCAL VOC data set, and can even outperform the full-precision model on real-world test images with complex visual scenes. Moreover, the deployment of our low-bit model appears to be more than 4$\times$ faster.  In future work, we plan to improve the low bit width models (especially the 4 bit-width model) further by exploring alternative training algorithms and refining our approximate quantization scheme.  

\subsubsection*{Acknowledgments}
This work was partially supported by NSF grants DMS-1522383 and IIS-1632935, and ONR grant N00014-16-1-7157.
 
\section*{Appendix}

\begin{proof}[Proof of Theorem \ref{thm:exact}] We first fix the number of entries in $Q$ quantized to $\pm2^{-t}$ to be $k_t$, $t = 0,\dots,n-1$. Then it is easy to show that
\begin{equation}\label{ineq}
\|Q\|^2 = \sum_{i=0}^{n-1} k_t 2^{-2t} \quad \mbox{and} \quad |\langle Q,W^f \rangle| \leq \sum_{t=0}^{n-1} \|W_{[k_t]}^f\|_1 2^{-t}.
\end{equation}
Therefore, for any $s\in\mathbb{Z}$,
\begin{align}\label{obj}
\|2^s Q-W^f\|^2&  = 2^{2s}\|Q\|^2 - 2^{s+1}\langle Q,W^f \rangle + \|W^f\|^2 \notag\\
& \geq 2^{2s}\sum_{i=0}^{n-1} k_t 2^{-2t} -2^{s+1}\sum_{t=0}^{n-1} \|W_{[k_t]}^f\|_1 2^{-t}  + \|W^f\|^2 \qquad (\mbox{by  (\ref{ineq})}) \notag \\
& = (\sum_{i=0}^{n-1} k_t 2^{-2t})\left(2^s- \frac{\sum_{t=0}^{n-1} \|W_{[k_t]}^f\|_1 2^{-t}}{\sum_{i=0}^{n-1} k_t 2^{-2t}}\right)^2 - \frac{(\sum_{t=0}^{n-1} \|W_{[k_t]}^f\|_1 2^{-t})^2}{\sum_{i=0}^{n-1} k_t 2^{-2t}}+ \|W^f\|^2\notag \\
\end{align}
Since $s\in\mathbb{Z}$, by symmetry of the parabola, it suffices to find the nearest power of 2 to $\frac{\sum_{t=0}^{n-1} \|W_{[k_t]}^f\|_1 2^{-t}}{\sum_{i=0}^{n-1} k_t 2^{-2t}}$. So the lower bound in (\ref{obj}) is achieved at $s = \floor*{\log_2 \frac{4\sum_{t=0}^{n-1} \|W_{[k_t]}^f\|_1 2^{-t}}{3\sum_{i=0}^{n-1} k_t 2^{-2t}}}$. Let us define $g(u,v):= v(2^{\log_2 \floor*{\frac{4u}{3v}}} - \frac{u}{v})^2 - \frac{u^2}{v}$. Then we examine the minimum value of $g(\sum_{t=0}^{n-1} \|W_{[k_t]}^f\|_1 2^{-t},\sum_{i=0}^{n-1} k_t 2^{-2t})$ over all possible combinations of natural numbers $k_0,\dots,k_{n-1}$, i.e., the optimal numbers of quantized weights at the $n$ levels are given by
$$
(k^*_0, \dots, k^*_{n-1}) = \arg\min_{k_0,\dots,k_{n-1}\in\N} \; g\left(\sum_{t=0}^{n-1} \|W_{[k_t]}^f\|_1 2^{-t}, \sum_{i=0}^{n-1} k_t 2^{-2t}\right)
$$
Finally, to achieve the minimum in (\ref{obj}) with respect to $(k^*_0, \dots, k^*_{n-1})$, we must have
\begin{equation*}\label{min_T}
Q^* = \sum_{t = 0}^{n-1}\sgn(W^f_{[k^*_t]})2^{-t}
\end{equation*}
so that $\langle Q^*,W^f \rangle = \sum_{t=0}^{n-1} \|W_{[k^*_t]}^f\|_1 2^{-t}$, and choose
$s^*=\floor*{\log_2 \frac{4\sum_{t=0}^{n-1} \|W_{[k^*_t]}^f\|_1 2^{-t}}{3\sum_{i=0}^{n-1} k^*_t 2^{-2t}}}$.

\end{proof}



\end{document}